\documentclass{article}
\usepackage{standalone} 
\usepackage{microtype}
\usepackage{graphicx}
\usepackage{subfigure}
\usepackage{booktabs} 
\usepackage{multirow}
\usepackage{amssymb}
\usepackage{amsmath}
\usepackage{amsthm}
\usepackage{pdfpages}
\usepackage[random]{blindtext}
\usepackage{hyperref}

\newtheorem{amp}{Assumption}[section]
\newtheorem{lma}{Lemma}[section]
\newtheorem{thm}{Theorem}[section]

\newtheorem{defn}{Definition}[section]
\newtheorem{prop}{Proposition}[section]

\usepackage[accepted]{icml2018}

\icmltitlerunning{Theoretical Analysis of Image-to-Image Translation with Adversarial Learning}

\begin{document}

\twocolumn[
\icmltitle{Theoretical Analysis of Image-to-Image Translation with Adversarial Learning}





\begin{icmlauthorlist}
\icmlauthor{Xudong Pan}{fudan}
\icmlauthor{Mi Zhang}{fudan}
\icmlauthor{Daizong Ding}{fudan}
\end{icmlauthorlist}

\icmlaffiliation{fudan}{Shanghai Key Laboratory of Intelligent Information Processing, School of Computer Science, Fudan University, China}

\icmlcorrespondingauthor{Mi Zhang}{mi\_zhang@fudan.edu.cn}

\icmlkeywords{Generative Adversarial Learning, Generalization, Theoretical Interpretations, Unsupervised Learning, Machine Learning}

\vskip 0.3in
]



\printAffiliationsAndNotice{} 

\begin{abstract}
Recently, a unified model for image-to-image translation tasks within adversarial learning framework \cite{Isola2017ImagetoImageTW} has aroused widespread research interests in computer vision practitioners. Their reported empirical success however lacks solid theoretical interpretations for its inherent mechanism. In this paper, we reformulate their model from a brand-new geometrical perspective and have eventually reached a full interpretation on some interesting but unclear empirical phenomenons from their experiments. Furthermore, by extending the definition of generalization for generative adversarial nets \cite{Arora2017GeneralizationAE} to a broader sense, we have derived a condition to control the generalization capability of their model. According to our derived condition, several practical suggestions have also been proposed on model design and dataset construction as a guidance for further empirical researches.
\end{abstract}
\section{Introduction}

Generative adversarial nets (GAN) \cite{Goodfellow2014GenerativeAN} have been a trending topic in machine learning community recent years, leading to a number of derived models \cite{Mirza2014ConditionalGA, Arjovsky2017WassersteinGA} and related theoretical works \cite{Arjovsky2017TowardsPM, Lei2017AGV}.
With wide and fruitful applications in various scenarios such as speech synthesis \cite{Saito2018StatisticalPS}, text generation \cite{Zhang2017AdversarialFM} and a considerable amount of visual tasks \cite{Denton2015DeepGI,Wu2016LearningAP,Kataoka2017AutomaticMC}, the idea behind GAN and its derivations is relatively simple and intuitive. It aims at learning a mapping from an artificial distribution, usually priorly known gaussian for original GAN and an unknown distribution of \textit{labels} for conditional GAN \cite{Mirza2014ConditionalGA}, to a real one. Via attaining an equilibrium of the minimax game \cite{aumann1989game} between a \textit{generator} (i.e. an adaptive model that maps a gaussian noise to a fake sample) and a \textit{discriminator} (i.e. an adaptive model to distinguish a fake sample from a real distribution), the adversarial learning models will finally learn a realistic distribution for further generative tasks  \cite{Arora2017GeneralizationAE}. 

Noticeably, last year has also witnessed an empirical success on traditional image-to-image translation tasks with the aid of a model under conditional GAN paradigm \cite{Isola2017ImagetoImageTW}, arousing widespread research interest in computer vision practitioners \cite{Zheng2017PhototoCaricatureTO,choi2017stargan}. Image-to-image translation, as a generic name for various specific tasks in image processing, includes tasks such as facial expression transfer (e.g. poker face $\to$ smile face), artistic style transfer (e.g. Van Gogh's $\to$ Monet's). Generally speaking, the goal of image-to-image translation is to process an image from a source collection to make it indistinguishable among a target collection of images. Although related models and methods abound in literature \cite{Reinhard2001ColorTB, Gatys2016ImageST, Zeiler2011FacialET}, the first attempt to tackle image-to-image translation as a whole instead of focusing on one of its specific tasks exclusively, ought to be attributed to the pioneering work of Isola et al. \yrcite{Isola2017ImagetoImageTW}, where the powerfulness of adversarial learning with conditional GAN has been once again exhibited sufficiently. We briefly review Isola's original model as the optimization problem below,  

\begin{equation} \label{eq:img2img_obj}
\min_{G}\max_{D} \mathcal{L}_{\text{cGAN}}(G,D) + \lambda \mathcal{L}_{L_1}(G)
\end{equation}
with $\mathcal{L}_{\text{cGAN}}$ the \textit{conditional GAN-loss} (or generally, \textit{adversarial loss}) defined as
\begin{align} \label{eq:gan_loss}
\mathcal{L}_{\text{cGAN}}(G,D) = \mathbb{E}_{x,y\sim{p_r(x,y)}}[\log{D(x,y)}] \nonumber \\
\qquad{} + \mathbb{E}_{x\sim{p_r(x)}}[1-\log{D(x,G(x))}] 
\end{align}
and $\mathcal{L}_{L_1}$ the \textit{L1 loss} (or \textit{identity loss}) as
\begin{equation}\label{eq:L1_loss}
\mathcal{L}_{L_1}(G) = \mathbb{E}_{x,y\sim{p_r(x,y)}}[\|y-G(x)\|_1] \\
\end{equation}
where $p_r(x,y)$ denotes the distribution of \textit{paired images} (e.g. in facial expression transfer, Bob's poker face and his ground-truth smile face) and $p_r(x)$ the distribution of images over the source collection, with $G, D$ respectively the generator and the discriminator, $\lambda$ the regularization factor.

As reported in their experiments \cite{Isola2017ImagetoImageTW}, several interesting but theoretically unclear results have attracted our attentions,
\begin{itemize}
\item Omitting the adversarial loss, i.e. solving $\mathcal{L}_{L_1}(G)$ alone, will "lead to reasonable but blurry results" (i.e. generated related target images, however with details hard to recognize), which we refer to as \textit{Blurry versus Sharp}.
\item Omitting the identity loss, i.e. setting the regularization factor $\lambda$ to 0, "gives much sharper results, but results in some artifacts" (i.e. generated realistic images however unrelated to the given source image), which we refer to as \textit{Source of Artifacts}.
\end{itemize}

Despite a number of studies devoted to analyzing and improving the training dynamics and generalization capability of GAN \cite{Arjovsky2017TowardsPM,Arora2017GeneralizationAE}, there is rarely applicable theoretical results for analyzing conditional GAN, thus Isola's original model and its empirical results. The inappropriateness mainly comes from Eq. \ref{eq:gan_loss}, where the model generates fake images directly from a given image of intensively high dimension \cite{Lu1998ImageM}, instead of a low-dimensional gaussian noise in GAN. In fact, the simple violation of the low-dimensional assumption would immediately invalidate most of the previously obtained theoretical results for GAN. Considering the worthiness of obtaining reasonable theoretical interpretations as guidance for further researches, we formulate this non-standard model from a geometrical perspective, propose an extended definition of generalization for conditional GAN and have eventually reached some inspiring theoretical results.

In this paper, for the convenience of mathematical manipulation, we will study a slightly different objective from Isola's original model (Eq. \ref{eq:img2img_obj}), by substituting the ordinary conditional GAN loss (Eq. \ref{eq:gan_loss}) with the Wasserstein GAN (WGAN) loss below as Eq. \ref{eq:cwgan_loss}. As is well known, such a replacement is usually adopted by experimenters to stabilize the model's training dynamics \cite{Arjovsky2017TowardsPM}.

\begin{equation} \label{eq:cwgan_loss}
\mathcal{L_{\text{adv}}}(G) =  \underset{\gamma\in{\Pi(p_r, p_g)}}{\inf}\mathbb{E}_{(x,y)\sim{\gamma}}[\|G(x)-y\|]
\end{equation}
 where $p_g$ denotes the distribution of images over the target collection, with $\Pi(p_r, p_g)$ the set of joint distributions for pairs of images $(x, y)$ s.t. the marginal distributions are equal to $p_r, p_g$. Note the explicit term of discriminator in GAN (Eq. \ref{eq:gan_loss}) is actually replaced by the inner optimal transport task \cite{Villani2010OptimalT} implicitly in WGAN loss (Eq. \ref{eq:cwgan_loss}).
 
Therefore, the corresponding objective can be formulated as

\begin{equation} \label{eq:img2img_obj_dev}
\overbrace{\min_{G} \underbrace{\mathcal{L}_{\text{adv}}(G)}_{\text{adversarial loss}} + \lambda \underbrace{\mathcal{L}_{L_1}(G)}_\text{identity loss}}^{\text{target model}}
\end{equation}

Aiming at exploring the intrinsic mechanism of our target model, we first formulate the image-to-image translation task with adversarial learning from a geometrical viewpoint (Section \ref{sec:geom_formulate}). With some basic results from topology and analysis, we have proved that the adversarial loss has an equivalent form (Eq. \ref{eq:final_equiv_form}), degenerated as a set of individual learning tasks between paired \textit{charts} (i.e. local neighborhoods homeomorphism to Euclidean space). We call such a result as \textit{natural localization of adversarial loss} (Theorem \ref{thm:natural_localization}). As a direct application of our theorem, theoretical interpretations has been presented fully for Source of Artifacts and partially for Blurry versus Sharp (Section \ref{sec:partial_interp}).

In order to explore the properties of our target model more quantitatively, we have extended the definition of generalization for GAN proposed by Arora et al. \yrcite{Arora2017GeneralizationAE} to a broader case for analysis of conditional GAN (Definition \ref{def:gen_wrt_gen}). We have then pointed out the relation between generalizations in different senses with a generic inequality for the first time as far as we know (Theorem \ref{thm:rel_generalization}) and have finally obtained the full picture of the mechanism behind Blurry versus Sharp (Section \ref{sec:full_interp}).

As a step further, we have derived a condition on controlling generalization for our target model with additional statistical settings (Theorem \ref{thm:generalization_cond}). As we will see, the derived inequality (Eq. \ref{eq:generalization_cond}) imposes concrete constraints on both the sample complexity and model complexity, which provides practical guidance on model design and dataset construction for further applications (Section \ref{sec:guidance}). Conclusions and future directions are provided in Section \ref{sec:conclude}.  

Generally, our contributions are outlined as follows,
\begin{enumerate}
 \setlength{\itemsep}{1pt}
\item A proposed geometrical formulation of image-to-image translation task with adversarial learning (Section \ref{sec:geom_formulate}) 
\item Reduction of the adversarial loss to a set of independent learning tasks between paired charts (Theorem \ref{thm:natural_localization})
\item An extended definition of generalization for conditional GAN (Definition \ref{def:gen_wrt_gen}) and a derived condition on generalization (Theorem \ref{thm:generalization_cond}) for our target model
\item Theoretical interpretations for several unclear empirical phenomenons reported in previous works (Section \ref{sec:partial_interp} \& \ref{sec:full_interp}), together with a practical guidance on model design and dataset construction for practitioners (Section \ref{sec:guidance})
\end{enumerate}

\section{Preliminaries} \label{sec:geom_formulate}
In Section \ref{sec:21} \& \ref{sec:22}, we equip a set of images with additional geometrical structures. In Section \ref{sec:23}, we correspondingly reformulate image-to-image translation with adversarial learning by extending the concept of generator and discriminator. A reformulation of our target model will thus be given in Eq. \ref{eq:geo_gan_loss} \& \ref{eq:geo_L1_loss} as a basis for analysis in the remainder of this paper.

\subsection{Set of Images as Smooth Manifold} \label{sec:21}
Without loss of generality, we mainly focus on the image-to-image translation task from a source set of RGB images $\mathcal{I}_S$ to a target set $\mathcal{I}_T$, with images of the same resolution $w\times{h}$. 
As is well-known, an image can always be considered as an element in an ambient Euclidean space (here, specifically $\mathbb{R}^{3\times{w}\times{h}}$). In fact, there also exists an intrinsic structure over the image set alongside with the ambient space, as is validated by various empirical works previously \cite{Lu1998ImageM, Zhu2016GenerativeVM}. Such an intrinsic structure is usually formulated as a smooth manifold mathematically \cite{ Arjovsky2017TowardsPM, Lei2017AGV}. For the basics of intrinsic geometry, see standard texts such as Lee's \yrcite{lee2010introduction}.

In this work, we make a similar assumption as follows,
\begin{amp} \label{amp:mfd}
There exist smooth, locally compact $d$-dimensional manifolds $\mathcal{M}$, $\mathcal{N}$ embedded in $\mathbb{R}^{w\times{h}}$, with constructed atlas as $\{(U_i,\varphi_i)\}_{i=1}^{K}$, $\{(V_i, \psi_i)\}_{i=1}^{K}$,  respecting pairwise disjointness property, i.e. $\forall{i,j}\in[K]$, $U_i\cap{U_j} = \varnothing$, $V_i\cap{V_j} = \varnothing$ if $i\neq{j}$, such that $\mathcal{I}_S\subset{\mathcal{M}}$, $\mathcal{I}_T\subset{\mathcal{N}}$. ($[K]$ denotes the set $\{1,2,\hdots, K\}$ and K a natural number)
\end{amp}

As a comment, the assumption of equal dimensions contained above is only for the convenience of notation simplification. Results presented in the remainder of this paper can be directly extended to the situation when source and target image manifolds are of different dimensions. 

\subsection{Induced Probability Measure on Image Manifold} \label{sec:22}
With Assumption \ref{amp:mfd}, we are able to divide the image sets $I_S\subset\mathcal{M}$, $I_T\subset{\mathcal{N}}$ into finer subsets, formally, that is 
$\mathcal{I}_S = {\cup}_{k=1}^{K}\mathcal{I}_{S}^{k}$, $\mathcal{I}_T = {\cup}_{k=1}^{K}\mathcal{I}_{T}^{k}$, where $\forall{k}\in[K]$, $\mathcal{I}_S^{k} \doteq \{s^{i}_k\}_{i=1}^{m_k}\subset{U_k}$ and $\mathcal{I}_T^{k} \doteq \{t^{i}_k\}_{i=1}^{n_k}\subset{V_k}$. 

In order to describe the relatedness of images from the same chart, the following assumption is imposed. 

\begin{amp} \label{amp:probability}
For each $k\in[K]$, there exists probability measures $\mu_{k}$, $\nu_{k}:\mathcal{B}(\mathbb{R}^{d})\to[0,1]$, supported on $\varphi_k(U_k)$ and $\psi_k(V_k)$ respectively, such that $\{\varphi_k(s^{i}_k)\}_{i=1}^{m_k}\overset{i.i.d.}{\sim}\mu_k$, $\{\psi_k(t^{i}_k)\}_{i=1}^{n_k}\overset{i.i.d.}{\sim}\nu_k$, where $\mathcal{B}(\mathbb{R}^d)$ denotes the Borel set over $\mathbb{R}^{d}$.
\end{amp}

With probability measures defined on each chart (precisely, its homeomorphism as $\mathbb{R}^{d}$), we would like to "glue" them together to induce a unified probability measure (denoted respectively as $\mu_\mathcal{M}, \nu_\mathcal{N}$) globally over the underlying manifold structure, with the aid of the following lemma.

\begin{lma} \label{lma:pm_over_mfd}
Given a smooth manifold $\mathcal{M} = \{(U_i,  \varphi_i)\}_{i=1}^{K}$ with pairwise disjointness and $\{\mu_i\}_{i=1}^{K}$ as the probability measures supported on $\{\varphi_i(U_i)\}_{i=1}^{K}$ correspondingly, a function $\mu_{\mathcal{M}}:\mathcal{B}(\mathcal{M})\to[0,1]$ is defined by
\begin{equation} \label{eq:pm_over_mfd}
	 d\mu_{\mathcal{M}}(s) = \frac{1}{K}\sum_{i=1}^{K}\mathbf{1}_{s\in{U_i}}d\mu_{i}\circ\varphi_i(s)
\end{equation}
Then $\mu_{\mathcal{M}}$ is a probability measure defined on $\mathcal{M}$. 
\end{lma}
\begin{proof}
See Appendix A. Although Definition \ref{eq:pm_over_mfd} seemingly contains some notation abusing (consider if $s\notin{U_i}$, $\varphi_i(s)$ is not defined), we can actually avoid this awkwardness according to the pairwise disjointness assumption, that is, all except one $\mathbf{1}\{s\in{U_i}\}$ is non-vanishing for any $s\in\mathcal{M}$.
\end{proof}

\subsection{A Geometrical Formulation of Image-to-Image Translation with Adversarial Learning} \label{sec:23}
After assuming additional geometrical structure on image set, the definition of generator and discriminator in our target model requires slight modifications correspondingly.

\textbf{On Generator} In our context, the generator should be redefined as a mapping between manifolds instead of flat Euclidean spaces. Formally, we denote the generator as $G:\mathcal{M}\to\mathcal{N}$, a measurable mapping w.r.t. $\mathcal{M}$, $\mathcal{N}$.

\textbf{On Discriminator} 
As we have pointed out, within the WGAN setting, the role of discriminator is correspondingly abdicated to the set of joint distributions $\Pi(p_r, p_g)$ and the norm $\|\bullet\|$ contained in Eq. \ref{eq:cwgan_loss}. However, the latter is usually not well-defined in manifold settings. As a natural way to make Eq. \ref{eq:cwgan_loss} \& \ref{eq:img2img_obj_dev} proper, we further equip the manifold structure $\mathcal{N}$ underlying the target set with a Riemmanian metric $\tau$  \cite{jost2008riemannian}, with a little more technical conditions for regularity. Eventually it comes to our formulation of the adversarial loss and the corresponding identity loss, with relatively minor modifications compared with Eq. \ref{eq:L1_loss} \& \ref{eq:img2img_obj_dev}.
\begin{align} \label{eq:geo_gan_loss}
\mathcal{L}_{adv}^{'}(G) = \underset{\gamma\in{\Pi(\mu_\mathcal{M}, \nu_\mathcal{N})}}{\inf}\mathbb{E}_{(s,t)\sim{\gamma}}d_{\mathcal{N}}(G(s),t)
\end{align}
\begin{equation}\label{eq:geo_L1_loss}
\mathcal{L}_{L_1}^{'}(G) = \mathbb{E}_{s,t\sim{p_r(s,t)}}d_{\mathcal{N}}(G(s),t) \\
\end{equation}
where $d_\mathcal{N}(\bullet, \bullet)$ denotes the $\tau$-induced geodesic distance on $\mathcal{N}$ \cite{jost2008riemannian}. For compatibility with inner-relatedness, we further assume $\forall{i\neq{j}\in[K]}$,  $\forall{x,y\in{V_i}, z\in{V_j}}$, $d_{\mathcal{N}}(x, y) \le d_\mathcal{N}(x,z)$.

\section{Natural Localization of Adversarial Loss}
A widely recognized difficulty for obtaining analytic solutions for adversarial loss lies in the nested optimization problem \cite{Goodfellow2014GenerativeAN} (here, specifically $\min_{G}\mathcal{L}_{\text{adv}}$).  In order to avoid such an obstacle, we will prove in this section that, within our proposed framework above, the inner infimum term in Eq. \ref{eq:geo_gan_loss} could be solved in closed form with non-trivial constraints on candidate set of generator $G$ (Theorem \ref{thm:natural_localization}). Furthermore, we have observed that the closed form has a decomposition as a set of independent learning tasks on \textit{paired charts} (i.e. a tuple of charts respectively of $\mathcal{M}, \mathcal{N}$, such as $(U_i, V_j)$), with the relations uniquely determined by the candidate sets (Eq. \ref{eq:final_equiv_form}). This result directly leads to theoretical interpretations fully for Source of Artifacts and partially for Blurry versus Sharp (Section \ref{sec:partial_interp}).

\subsection{An Equivalent Form of $\mathcal{L}_{adv}^{'}$}
We start our derivation by giving the explicit form of the probability measures $\mu_{\mathcal{M}}$, $\nu_{\mathcal{N}}$ on manifolds, constructed with the aid of Lemma \ref{lma:pm_over_mfd}.
\begin{equation} \label{eq:mu_form}
	 d\mu_{\mathcal{M}}(s) \doteq \frac{1}{K}\sum_{i=1}^{K}\mathbf{1}_{s\in{U_i}}d\mu_{i}\circ\varphi_i(s)
\end{equation}
\begin{equation}\label{eq:nu_form}
	 d\nu_{\mathcal{N}}(t) \doteq \frac{1}{K}\sum_{i=1}^{K}\mathbf{1}_{t\in{V_i}}d\nu_{i}\circ\psi_i(t)
\end{equation}
For simplicity, we denote $d\tilde{\mu}_i = d\mu_{i}\circ\varphi_i$ and $d\tilde{\nu}_i = d\nu_{i}\circ\psi_i$, $\forall{i}\in{[K]}$.

We then expand $\mathcal{L}_{adv}^{'}$ with assumed pairwise disjointness property and obtain
\begin{align}\label{eq:decomp_form}
\underset{\gamma\in{\Pi(\mu_\mathcal{M}, \nu_\mathcal{N})}}{\inf}\mathbb{E}_{(s,t)\sim{\gamma}}\sum_{i=1}^{K}\sum_{j=1}^{K}\mathbf{1}_{s\in{U_i}}\mathbf{1}_{t\in{V_j}} d_{\mathcal{N}}(G(s),t)
\end{align}
By exchanging the expectation operator with summations according to Fubini's theorem \cite{Rudin2010REALAC} and writing the expectation directly in integral form, we have
\begin{align} \label{eq:interm_form}
\underset{\gamma\in{\Pi(\mu_\mathcal{M}, \nu_\mathcal{N})}}{\inf}\sum_{i=1}^{K}\sum_{j=1}^{K}\int_{U_i}\int_{V_j} d_{\mathcal{N}}(G(s),t)d{\gamma(s,t)}
\end{align}

With a similar technique adopted in Dai et al. \yrcite{Dai2008TranslatedLT}, for every $\gamma\in\Pi(\mu_\mathcal{M}, \nu_\mathcal{N})$, there exist functions $\Delta:\mathcal{N}\times{\mathcal{N}}\to\mathbb{R}^{+}\cup\{0\}$ and $f_{\gamma}:\mathcal{M}\to\mathcal{N}$, satisfying
\begin{align} \label{eq:jt_rewritten}
		d{\gamma(s,t)} = d{\gamma(t|s)}d{\mu_{\mathcal{M}}(s)} = \Delta(f_{\gamma}(s),t)d{\mu_{\mathcal{M}}(s)}d{\nu_{\mathcal{N}}(t)}
\end{align}

where $\Delta$ has an intuitive interpretation as a metric of similarity between elements on manifold $\mathcal{N}$, usually independent of the choice of path and compatible with inner-relatedness. Recall $\mathcal{N}$ as a Riemmanian manifold is naturally equipped with a 'divergence' metric $\tau$. We claim it is proper to absorb the term $\Delta(f_{\gamma}(s),t)$ into $d_{\mathcal{N}}(G(s),t)$ with the following observations.
\begin{itemize}

\item [a)] Equivalence of optimization problems (without boundary condition) \cite{boyd2004convex}
\begin{itemize}
		\item $\min_{G} \min_{f_{\gamma}}\Delta(f_{\gamma}(s),t)d_{\mathcal{N}}(G(s),t)$
       \item $\min_{G}\Delta(G(s),t)d_{\mathcal{N}}(G(s),t)$ 
\end{itemize} considering the relatively large learning capacity of G, usually implemented as a neural network \cite{Sontag1998VCDO}.
 
 \item [b)] It is possible to alter the choice of the original metric $\tau$ to be the metric $\tau^{'}$ induced by a modified distance function $d_{\mathcal{N}}^{'}(\bullet, \bullet) = \Delta(\bullet, \bullet)d_\mathcal{N}(\bullet, \bullet)$, which is asserted by the following lemma. 
\begin{lma} \label{lma:absorption}
Consider Riemmanian manifold $(\mathcal{N}, \tau)$ with curvature locally bounded above and below, $\tau \in C^{\infty}$ and its induced distance function denoted as $d_\mathcal{N}$, then for any path-independent function $f:\mathcal{N}\times{\mathcal{N}}\to\mathbb{R}^{+}\cup\{0\}$, there exists a Riemmanian metric $\tau^{'}$ on $\mathcal{N}$,  induced by the distance function
$$
d_{\mathcal{N}}^{'}(x,y) = f(x,y)d_\mathcal{N}(x,y)\quad{}\forall{x,y}\in{\mathcal{N}}
$$
\end{lma}
\begin{proof}
See Appendix A.
\end{proof}
\end{itemize}
After we rearrange $d_\mathcal{N}(G(s), t)d\gamma$ as $d_\mathcal{N}^{'}(G(s), t)d\gamma{'}$, the boundary condition $\int_{\mathcal{M}}\int_{\mathcal{N}}d\gamma^{'}=1$ requires renormalization. By introducing an additional matrix $A\in\mathbf{H}(K)$ s.t. $\mathbf{H}{(K)}\doteq\{A\in\mathbb{R}^{K\times{K}} | \forall{j}\in{K},\sum_{i}A^{ij} = K$; $\forall{i,j}\in[K], A^{ij} \ge 0\}$, the adversarial loss $\min_G\mathcal{L}^{'}_{adv}(G)$ can be reformulated as (following Eq. \ref{eq:mu_form}, \ref{eq:nu_form}, \ref{eq:interm_form}, with details in Appendix A)
\begin{gather} \label{eq:equiv_form}
	\min_{G} \min_{A\in \mathbf{H}(K)} \sum_{i=1}^{K}\sum_{j=1}^{K}\int_{U_i}\int_{V_j}A^{ij} d^{'}_{\mathcal{N}}(G(s),t)d\tilde{\mu}_{i}(s)d\tilde{\nu}_{j}(t)
\end{gather}

\subsection{Closed-Form Solution as Learning Tasks on Paired Charts}
As is discussed above, the form of Eq. \ref{eq:equiv_form} basically comes from a re-choice of Riemmanian metric on $\mathcal{N}$ and a reparametrization of $d\gamma(s,t) $ as $\sum_{i=1}^{K}\sum_{j=1}^{K}A^{ij}d\tilde{\mu}_{i}(s)d\tilde{\nu}_{j}(t)$, s.t. $A\in\mathbf{H}(K)$. Although it is almost infeasible to obtain a closed form solution for arbitrary mapping $G$, we are curious of the possibility by imposing non-trivial constraints on the candidate set of G. In our approach, we first propose the following definition.
\begin{defn} [pairwise topological immersion family (PTI-family)] \label{def:pti_family}
	Given topological manifolds $\mathcal{M} = \{(U_i, \varphi_i)\}_{i=1}^{K}$ and $\mathcal{N} = \{(V_i, \psi_i)\}_{i=1}^{K}$, the set of mappings $F_p = \{G:\mathcal{M}\to\mathcal{N}|G(U_i) \subset V_{p(i)}, \forall{i}\in[K]\}$, where $p\in\text{Sym}(K)$ the symmetric group of $[K]$ \cite{cameron1999permutation}, is called pairwise topological immersion mappings indexed by $p$, w.r.t $\mathcal{M}$, $\mathcal{N}$.
\end{defn}

Postponing remarks on this definition (Section \ref{sec:remark_on_def}), a main result of this paper will be provided subsequently, which shows that, we can indeed obtain a meaningful closed form of solution for the inner minimization problem, by constraining the candidate set of $G$ as an arbitrary PTI-family (Def. \ref{def:pti_family}).

\begin{thm} \label{thm:natural_localization} [Natural Localization of Adversarial Loss]
For any $p\in\text{Sym(K)}$, the optimization problem below (compared with Eq. \ref{eq:equiv_form})
\begin{gather} \label{eq:constrained_equiv_form}
	\min_{G\in F_p} \min_{A\in\mathbf{H}(K) } \sum_{i=1}^{K}\sum_{j=1}^{K}\int_{U_i}\int_{V_j}A^{ij} d^{'}_{\mathcal{N}}(G(s),t)d\tilde{\mu}_{i}(s)d\tilde{\nu}_{j}(t)
\end{gather}
is equivalent to 
\begin{gather} \label{eq:final_equiv_form}  
\min_{G\in F_p} \sum_{i=1}^{K}\int_{U_i}\int_{V_{p(i)}}d^{'}_{\mathcal{N}}(G(s),t)d\tilde{\mu}_{i}(s)d\tilde{\nu}_{p(i)}(t)
\end{gather}
In other words, the optimal $A^{*}\in\mathbf{H}(K) $ has the closed form as $
(A^{*})^{ij} = K\delta_{j}^{p(i)} $, where $\delta_{j}^{p(i)}$ is the Kronecker delta function.
\end{thm}
\begin{proof}[Sketch of Proof]
Fix $i, j\in[K]$, s.t. $j\neq{p(i)}$ and arbitrary $G\in{F_p}$. The key step lies in comparing the terms (remember the positiveness of $A^{ij}$)
$$
	T_{\text{non-paired}} = \int_{U_i}\int_{V_j}d^{'}_{\mathcal{N}}(G(s),t)d\tilde{\mu}_{i}(s)d\tilde{\nu}_{j}(t)
$$
and 
$$
	T_{\text{paired}} = \int_{U_i}\int_{V_{p(i)}}d^{'}_{\mathcal{N}}(G(s),t)d\tilde{\mu}_{i}(s)d\tilde{\nu}_{p(i)}(t)
$$
For any $s\in{U_i}$, $G(s)\in{V_{p(i)}}\cap{V_j} = \varnothing$, which leads to $\forall{t\in{V_{p(i)}}, t^{'}\in{V_j}}$, $d_{\mathcal{N}}(G(s), t)\le{d}_{\mathcal{N}}(G(s),t^{'})$. And thus $T_{\text{non-paired}} \ge T_{\text{paired}}$. A rigorous proof can be found in Appendix A.
\end{proof}

\subsection{Discussions \& Interpretations} \label{sec:partial_interp}
The final part of this section is devoted to a discussion on what Definition \ref{def:pti_family} and Theorem \ref{thm:natural_localization} actually mean, together with their significant roles in interpretations for the empirical phenomenons. 
\subsubsection{Discussion with an illustrative example} \label{sec:remark_on_def}
Intuitively, we may consider each chart on $\mathcal{M}$, $\mathcal{N}$ as a cluster of images, which has inner-relatedness imposed by $\{\mu_{i}\}_{i=1}^{K}$, $\{\nu_i\}_{i=1}^{K}$. For example, in facial expression translation tasks \cite{choi2017stargan}, $U_i$ contains a set of Bob's poker face, while $V_j$, $V_k$ are respectively sets of Alice's and Bob's face with smile. A PTI-family $F_p$ exactly characterizes the generating tendency of a given generator $G$. Let us come back to the example. Fix $i\neq{j}\neq{k}$. Assume $p, q \in Sym(K)$ with $p(i) = j$ while $q(i) = k$. Thus with the input as an image of Bob's poker face, generators from $F_p$ tends to generate a sample of Alice's smiling face, while those from $F_q$ prefer a sample of smiling Bob. Note that, although it is clear to us the latter behavior is expected, the adversarial learning model itself however hardly has such a knowledge.

It comes to the significance of Theorem \ref{thm:natural_localization}, which is not just an intention to give a closed form for further analysis. More essentially, such a theorem points out the role of $\{F_p\}_{p\in{\text{Sym(K)}}}$ as \textit{attractors} (for attractors in a general sense, see Luenberger's \yrcite{luenberger1979introduction}) during optimization. As we can see, only if the optimizer chooses some generator $G\in{F_p}$ at some epoch, the original optimization problem (Eq. \ref{eq:constrained_equiv_form}) will immediately degenerate to learning tasks on paired charts $\{(U_i, V_{p(i)})\}_{i=1}^{K}$ (Eq. \ref{eq:final_equiv_form}). The generator will thus be trapped in the subset $F_p$ until the end of the training. This theorem can be considered as a support to a recent result called \textit{imaginary adversary}, which points out that in WGAN setting, the minimax game between generator and discriminator can be resolved under some technical conditions \cite{Lei2017AGV}.

\subsubsection{Partial Interpretations for Empirical Results}
\textbf{Source of Artifacts}
Although it brings sharper results with the adversarial loss, a non-negligible proportion of artifacts is observed in experiments \cite{Isola2017ImagetoImageTW,choi2017stargan}. As a reasonable interpretation, we suggest it is tightly related with what we have discussed above. Since the adversarial learning model itself has no knowledge of the expected pairing relation, or formally the ground-truth $p\in{\text{Sym}(K)}$. Although the choice of $G$ (thus $F_p$) can be guided by the empirical loss during the training phase, it still has a large probability to mistake. Especially when the optimal pairing it observes is different from the expected one, a PTI-family as an attractor will let the choice irrevocable. A clever approach is by imposing oracle as a regularization term with L1-loss (Eq. \ref{eq:geo_L1_loss}), which plays the role as a \textit{rectifier} for choice of $p$. 

\textbf{Blurry versus Sharp}
In previous empirical studies, after learning with identity loss (Eq. \ref{eq:geo_L1_loss}) alone, the final generator usually produces more blurry images compared with the generator after learning with the adversarial loss (Eq. \ref{eq:geo_gan_loss}). When both of the losses are optimized w.r.t. the same hypothesis space, the identity loss needs to learn a global mapping $G^{*}:\mathcal{M}\to\mathcal{N}$, while, as a direct result of Theorem \ref{thm:natural_localization}, learning with adversarial loss theoretically only requires to learn independent local mappings $\{f_{i}:U_i \to V_{p(i)}\}_{i=1}^{K}$ first and then gluing them into a global mapping with a well-known theorem from general topology called \textit{partition of unity} \cite{Rudin2010REALAC}. Intuitively, learning local mappings independently requires much smaller capacity of $G$, compared with learning a globally compatible one (a theoretical justification, see Proposition \ref{prop:beau_local}). 
Recently, a model with a similar consideration by \textit{artificially} localizing the adversarial loss to improve the generating quality was proposed \cite{Qi2017GlobalVL}. However, their work mainly targets on image generation (i.e. only the target manifold structure is considered) and stays on an empirical level, with little theoretical analysis for the inherent mechanism. 

As a complement and a step further, we will provide a formal analysis on the benefit of localization detailedly (Section \ref{sec:full_interp}) to complete our interpretations. Due to the indispensable role of the concept of generalization in analyzing model's learning capability \cite{vapnik1998statistical}, we will first present an extended definition of generalization for conditional GAN in the next section.

\section{Generalization for Conditional GAN}
\subsection{Extension from Previous Definition}
As generalization plays a central role in analyzing learning models from a theoretical aspect, there have been previous efforts on proposing specific definitions for GAN considering its difference from conventions. One of these definitions is provided as follows, with our notations.

\begin{defn} \cite{Arora2017GeneralizationAE} [Generalization w.r.t Divergence] \label{def:gen_wrt_dist}
A divergence $D(\bullet,\bullet)$ is said to \textit{generalize} with $m$ training samples and error $\epsilon$ if for the learned distribution $\nu_{\mathcal{N}}$, the following inequality holds with high probability, 
\begin{equation} \label{eq:gen_wrt_dist}
|D(\hat{\nu}_{\text{real}}, \hat{\nu}_{\text{G}}) - D(\nu_{\text{real}}, \nu_{\text{G}})| < \epsilon
\end{equation}
where $\hat{\nu}_{\text{real}}, \hat{\nu}_{\text{G}}$ are empirical versions of real and generated distributions with $\nu_{\text{real}}$ the real distribution as ground-truth.
\end{defn}

Although their work marks the first attempt to study the generalization capability of GAN, such a definition has several potential shortcomings: \textbf{a)} generalization is defined w.r.t specific divergence, instead of the generator itself. From our perspective, it is still the generator that holds the fundamental position in generative tasks. \textbf{b)} lack of the extensibility to conditional GAN, which however plays an increasingly significant role in empirical research and applications. Such a deficiency directly makes it improper to be applied to analyze our target model.

In order to alleviate these possible downsides, we propose an extended version of generalization for both GAN and its deviations with respect to a learned generator.

\begin{defn} [Generalization w.r.t Generator] \label{def:gen_wrt_gen}
Given a divergence $D(\bullet,\bullet)$ and a generator $G:\mathcal{M}\to\mathcal{N}$, we call $G$ \textit{generalizes} with $(m, n)$ training samples respectively from source (or condition) and target distributions and error $\epsilon$ if the following inequality holds with high probability,
\begin{equation} \label{eq:gen_wrt_gen}
D(G(\hat{\mu}_{\mathcal{M}}^m), \nu_{\mathcal{N}}) - D(\hat{\nu}_{\mathcal{N}}^{n}, \nu_{\mathcal{N}}) < \epsilon
\end{equation}
where $\hat{\mu}_{\mathcal{M}}^{m}, \hat{\nu}_{\mathcal{N}}^{n}$ are estimators of source and target distributions, with $\mu_{\mathcal{M}}, \nu_{\mathcal{N}}$ the corresponding ground-truth distributions and $G(\hat{\mu}_{\mathcal{M}}^{m}) \doteq \hat{\mu}_{\mathcal{M}}^{m}\circ{G^{-1}}$, the induced distribution  on $\mathcal{N}$ \cite{chung2001course}.
\end{defn}

Compared with Definition \ref{def:gen_wrt_dist}, our extension explicitly contains the generator as an essential factor for generalization. Furthermore, instead of assuming the source distribution as a gaussian priorly known, we depict it with an empirical estimator from $m$ observed samples. Notice our definition is actually an extension of Definition \ref{def:gen_wrt_dist}, since, by limiting $m$ to infinity and assuming $G$ of sufficient learning capability (in a classical sense), Inequality \ref{eq:gen_wrt_gen} will directly degenerate to Inequality \ref{eq:gen_wrt_dist} in the previous definition.

\subsection{Relations of Generalization in Different Senses}
As an auxiliary theorem for further analysis of our target model in the next section, we will derive the relation of generalization in different senses as well. 

We first specify the divergence $D(\bullet, \bullet)$ in Definition \ref{def:gen_wrt_gen} as Lukaszyk-Karmowski metric \cite{lukaszyk2004new}
\begin{equation} \label{eq:lk_metric}
	D_{\text{LK}}(\nu,\nu^{'}) = \int_{\mathbb{R}^{d}}\int_{\mathbb{R}^{d}} \|x - x^{'}\| d\nu(x)d\nu^{'}(x^{'})
\end{equation}
where $\nu,\nu^{'}$ are arbitrary probability measures supported on $\mathbb{R}^{d}$ (compared with Eq. \ref{eq:final_equiv_form}). Note the Euclidean form above brings convenience for analysis and it actually only requires several technical steps to extend the following result to an intrinsic form (Lemma \ref{lma:norm_equiv}).

\begin{thm} \label{thm:rel_generalization}
Consider generator $G:\mathbb{R}^{d}\to\mathbb{R}^{d}$ satisfying Lipschitz condition with constant $M_{G}$ and $\mu_X,\nu_Y$ are probability measures on $\mathbb{R}^{d}$ respectively with $\{x_{i}\}_{i=1}^{n_X}\overset{\text{i.i.d.}}{\sim}\mu_{X}$ and $\{y_{i}\}_{i=1}^{n_Y}\overset{\text{i.i.d.}}{\sim}\nu_{Y}$. 

Assume the classical generalization bound satisfies the following inequality with probability $1-\delta$
\begin{align} \label{eq:class_gen}
		\mathbb{E}_{x\sim\mu_X, y\sim\nu_Y}{\|G(x) - y\|} - \sum_{i=1}^{n_X}\sum_{j=1}^{n_Y}\frac{\|G(x_i) - y_j\|}{n_Xn_Y} < \epsilon_{\text{classical}} 
\end{align}
where $\epsilon_{\text{classical}} \doteq \epsilon(n_X, n_Y, \mu_X, \nu_Y, \delta)$ the upper bound 
and ERM-principle \cite{vapnik1998statistical} is satisfied with $\eta$ (i.e. $\frac{1}{n_Xn_Y}\sum_{i=1}^{n_X}\sum_{j=1}^{n_Y}\|G(x_i) - y_j\|  < \eta$), then G generalizes with $(n_X, n_Y)$ training samples and error $\epsilon_{\text{adv}}$ with probability $1-\delta$, i.e. 
\begin{equation} 
	D_{LK}(G(\hat{\mu}_{X}^{n_X}), \nu_Y) - D_{LK}(\hat{\nu}_{Y}^{n_Y}, \nu_Y) < \epsilon_{\text{adv}}
\end{equation}
if the following condition is satisfied
\begin{equation} \label{eq:suff_gen_cond}
\epsilon_{\text{classical}}  - \epsilon_{\text{adv}} + \eta < D_{LK}(\nu_Y,\hat{\nu}_{Y}^{n_Y}) - M_{G}D_{LK}(\mu_{X}, \hat{\mu}_{X}^{n_X}) 
\end{equation}
\end{thm}
\begin{proof}
See Appendix A.
\end{proof}
As Theorem \ref{thm:rel_generalization} indicates, unlike the classical generalization bound (especially in VC sense \cite{vapnik1998statistical}), the generalization error in adversarial learning is also affected by the variation of distributions in source and target distributions.

\section{Benefits of Localization and Conditions of Generalization}
By auxiliary of the extended definition of generalization above, we are now able to complete our interpretations for Blurry versus Sharp (Section \ref{sec:full_interp}). As a step further, we will derive a concrete condition (Theorem \ref{thm:generalization_cond}) to control the generalization capability of our target model, which will directly provide practical guidance on model design and dataset construction for practitioners.

We start by specifying some additional statistical settings, only for the sake of concreteness. Recall in Assumption \ref{amp:probability}, we have imposed abstract probability measures $\{\mu_i\}_{i=1}^{K}$, $\{\nu_i\}_{i=1}^{K}$ on $\{\varphi_{i}(U_i)\}_{i=1}^{K}$ and $\{\psi_{i}(V_i)\}_{i=1}^{K}$ respectively. We further specify such an assumption with gaussian settings \textit{locally}.
\begin{amp} 
There exist unknown mean vectors in $\mathbb{R}^{d}$, denoted as $\{x_{i}\}_{i=1}^{K}$, $\{y_{i}\}_{i=1}^{K}$, and known covariance matrices $\Sigma_{\mathcal{M}}, \Sigma_{\mathcal{N}}\in\mathbb{R}^{d\times{d}}$, such that for each $i\in[K]$, $\mu_i = \mathcal{N}(\bullet;x_i, \Sigma_{\mathcal{M}})$, $\nu_i = \mathcal{N}(\bullet;y_i, \Sigma_{\mathcal{N}})$, where $\mathcal{N}(\bullet; x, \Sigma)$ denotes the normal distribution parametrized by $(x, \Sigma)$. Additionally, we set the sample sizes on charts $\{U_i\}_{i=1}^{K}$, $\{V_i\}_{i=1}^{K}$ equally as $m$, $n$, without loss of generality.
\end{amp}

It ought to be noticed that our gaussian assumption above will not impose much limitation on our discussion, mainly because its influence remains local (compared with original GAN \cite{Goodfellow2014GenerativeAN}) and each gaussian is partially unknown (compared with Arora et al. \yrcite{Arora2017GeneralizationAE}).

\subsection{Benefits of Localization} \label{sec:full_interp}
In our previous interpretation for Blurry versus Sharp (Section \ref{sec:partial_interp}), a claim has remained unjustified that learning a set of local mappings is much easier compared with learning a globally compatible one. With the following observations: \textbf{a)}  Lipschitz condition can be always satisfied with techniques during training phase \cite{Arjovsky2017WassersteinGA}. \textbf{b)} $\epsilon_{\text{classical}}, \eta, M_G$ remain constant for the same hypothesis space. \textbf{c)} The target-related term $D_{LK}(\nu, \hat{\nu})$ is identical in local and global task when the pairing relation is unobserved,  
we reformulate Inequality \ref{eq:suff_gen_cond} as 
\begin{equation} \label{eq:suff}
C + \lambda{D}_{LK}(\mu,\hat{\mu}) < \epsilon_{\text{adv}}
\end{equation}
where $C \doteq \epsilon_{\text{classical}} + \eta - D_{LK}(\nu, \hat{\nu})$ a constant and $\lambda \doteq{M_G} > 0 $.

By denoting probability measures underlying the global task as $\mu_X = \frac{1}{K}\sum_{i=1}^{K}\mu_i$ and $\nu_Y = \frac{1}{K}\sum_{i=1}^{K}\nu_i$ in Euclidean sense, it is sufficient to compare the two terms below to justify our previous claim. 
\begin{equation}
		\epsilon^{local}_{adv} = \frac{1}{K}\sum_{i=1}^{K}{D}_{LK}(\mu_{i},\hat{\mu}_{i}^{m})
\end{equation}
\begin{align}
	\epsilon^{global}_{adv} ={D}_{LK}(\frac{1}{K}\sum_{i=1}^{K}\mu_i, \hat{\mu}_{X}^{Km})
\end{align}
Intuitively, the term $\epsilon^{local}_{adv}$ represents the average generalization errors for all the local tasks ($\mu_i\to\nu_i$, $\forall{i}\in[K]$), while $\epsilon^{global}_{adv}$ can be interpreted as the generalization error when the learning process is carried out globally ($\mu_X\to\nu_{Y}$). Note, for convenience, we have set the pairing relation $e\in{Sym(K)}$ as $e(i) = i, \forall{i}\in[K]$ (the corresponding PTI-family denoted as $F_e$). With the following proposition, we have eventually completed our unfinished interpretations for the empirical results.

\begin{prop} \label{prop:beau_local}
In the settings above, we always have
$$
	\epsilon^{local}_{adv} < \epsilon^{global}_{adv}
$$
\end{prop}
\begin{proof} See Appendix A. Intuitively, let us consider an extremal situation when $\mu_i = \delta_{x_i}$, $\nu_i = \delta_{y_i}$ and $m\to{\infty}$ ($\delta_x$ is the Dirac function as a distribution) for each $i\in{[K]}$. Thus $\epsilon^{local}_{adv} = 0$, while $\epsilon^{global}_{adv} \ge \sup_{i,j\in[K]}\|x_i - x_j\| > 0$.
\end{proof}

\subsection{Conditions of Generalization} \label{sec:guidance}
As a step further, we would like to derive some technical conditions under which the target model will generalize well, in the sense of adversarial learning (Definition \ref{def:gen_wrt_gen}).

In order to apply Theorem \ref{thm:rel_generalization}, we introduce the following lemma which points out the equivalence between L-K metric (Eq. \ref{eq:lk_metric}) with assumed probability measures in Euclidean space and each local objective defined on manifolds.

\begin{lma} \label{lma:norm_equiv}
$\forall{i}\in[K]$, consider a measureable mapping $\tilde{f}:U_i\to{V_i}$ with $f \doteq \psi_{i}\circ{\tilde{f}}\circ\varphi_{i}^{-1}$ satisfies Lipschitz condition, then $\int_{U_i}\int_{V_i}d^{'}_{\mathcal{N}}(\tilde{f}(s),t)d\tilde{\mu}_{i}(s)d\tilde{\nu}_{i}(t) \simeq D_{LK}(f(\mu_i), \nu_i)$, i.e. there exists constants $0< C_l < C_u < \infty$ such that 
\begin{equation}
	C_l < \frac{\int_{U_i}\int_{V_i}d^{'}_{\mathcal{N}}(\tilde{f}(s),t)d\tilde{\mu}_{i}(s)d\tilde{\nu}_{i}(t)}{D_{LK}(f(\mu_i), \nu_i)} < C_u
 \end{equation}
\end{lma}
\begin{proof} [Sketch of Proof]
The key observation lies in, with the measureability of $\tilde{f}$ and smoothness of $\varphi_i, \psi_i$, the induced mapping $\tilde{\nu}^{'} = \frac{1}{\tilde{E}}\tilde{f}(\mu_i)$ and $\nu^{'} = \frac{1}{E}(\psi_i\circ\tilde{f})(\mu_i)$ are also probability measures respectively on $\tilde{f}(U_i)\subset{V_i}$ and $(\psi_i\circ\tilde{f})(U_i)\subset{\psi_i(V_i)}$ (with $\tilde{E},E$ some normalizing factor), with bounded range a.e., which basically comes from the Lipschitz condition. We also use the assumption $\tilde{f}(U_i) \subset V_i$ and $tr(\Sigma_\mathcal{N}) < \infty$. For the rest of the proof, see Appendix A.
\end{proof}

We are now able to instantiate the generic inequality on generalization in the form of the following theorem. Note we depict the classical generalization error term in VC sense and study the condition for $\epsilon_{adv} = 0$, which means the generated distribution is even better than an estimated target distribution from $m$ real samples. 
\begin{thm} \label{thm:generalization_cond}
Under the assumptions above, consider a generator $G \in F_e$ and a hypothesis space $\mathcal{H}$ with VC-dimension bound by constant $\Lambda$. Assume for each $i\in[K]$, the restriction of G to a pair of charts $f_i \doteq G_{\downarrow{(U_i, V_i)}}\in\mathcal{H}$ with $\psi_{i}\circ{G}\circ\varphi_{i}^{-1}$ satisfies Lipschitz condition with constant $M_G$, then G generalizes globally with $(Kn, Km)$ samples only if the following inequality is satisfied with probability $1- C(\epsilon, \Lambda)(nm\epsilon^2)^{\tau(\Lambda)}e^{-nm\alpha\epsilon^{2}}$,  
\begin{gather} 
	\epsilon + \frac{1}{nm}\max\{\sum_{i=1}^{n}\sum_{j=1}^{m}d_\mathcal{N}(G(s_{k}^{i}), t_{k}^{j})\}_{k=1}^{K} < \nonumber \\
   \label{eq:generalization_cond}
   	\frac{1}{\sqrt{m}}\sqrt{\text{tr}(\Sigma_\mathcal{N})} + 2\text{tr}(\Sigma_\mathcal{N}) - M_G(\frac{1}{\sqrt{n}}\sqrt{\text{tr}(\Sigma_\mathcal{M})} + 2\text{tr}(\Sigma_\mathcal{M}))
\end{gather}
where $C(\epsilon, \Lambda)$ and $\tau(\Lambda)$ are positive functions independent from $n,m$ and $\alpha\in[1,2]$ a constant.   
\end{thm}
\begin{proof} See Appendix A. Besides Theorem \ref{thm:rel_generalization} \& Lemma \ref{lma:norm_equiv}, we have applied a general form of Vapnik-Chervonenkis theorem \cite{Vayatis1999DistributionDependentVB} for worst case analysis and a non-asymptotic theorem from information geometry as follows,
\begin{thm} \cite{amari2007methods}
The mean square error of a biased-corrected first-order efficient estimator $\hat{u}$ to $\mu$ is given by the expansion (with N observed samples):
$$
	\mathbb{E}[(\hat{u}^{a} - u^{a})(\hat{u}^{b} - u^{b})] = \frac{1}{N}g^{ab} + O(\frac{1}{N^2})
$$
where $g^{ab}$ denotes the Fisher metric on the statistical manifold underlying a parametrized family of probability. 
\end{thm}

\end{proof} 

\textbf{Discussions \& Guidance for Practitioners}
A brief discussion on Theorem \ref{thm:generalization_cond} and its possible guidance on practice will serve as the last topic. As we can see, generalization happens with a higher probability when the right side of Inequality \ref{eq:generalization_cond} yields larger and the left side becomes smaller. The former situation corresponds to a smaller variance of each local source distribution, especially when $M_G$ the Lipschitz constant lets the $\text{tr}(\Sigma_\mathcal{M})$ term dominate. The latter situation corresponds to a \textit{uniformly} lower empirical risk. As each local chart has an intuitive interpretation as a set of related images, it is reasonable to make the following suggestions on dataset construction and model design.

\begin{itemize}
\item The source set of images should be of lower inner-similarity, i.e. a set of N different individuals' poker face will give a better generator rather than a set of N different photos of the same person's poker face.
\item A blind increase in total number of images will hardly help generalization, while the balancedness in numbers of different objects is what actually matters.
\item Classical generalization capacity \cite{vapnik1998statistical} and smoothness of learning model w.r.t. data manifolds \cite{Belkin2006ManifoldRA} should be considered equivalently important in model design for such tasks.
\end{itemize}

\section{Conclusion and Further Directions} \label{sec:conclude}
In this paper, we have focused on providing a solid theoretical interpretations for some critical but unclear empirical phenomenons reported in Isola et. al \yrcite{Isola2017ImagetoImageTW}. Via reformulating Isola's model within a brand-new geometrical framework (Section \ref{sec:geom_formulate}), we have proved that the target model has a natural localized form as independent learning tasks on paired charts (Theorem \ref{thm:natural_localization}), which directly provides a candidate interpretation for their experimental results (Section \ref{sec:partial_interp}). Furthermore, with our extension of the generalization concept for GAN to conditional GAN case (Definition \ref{def:gen_wrt_gen}), we have successfully described the inherent mechanism of the target model in a full picture (Section \ref{sec:full_interp}). Our derived generalization condition (Theorem \ref{thm:generalization_cond}) also provides constructive guidance for further empirical studies (Section \ref{sec:guidance}). 

Actually, our theoretical results can be easily decoupled from the image-to-image translation setting to a much general case, that is, learning translation from a source manifold structure to a target one via adversarial learning. Further directions in applications, such as applying our theoretical results for improving the current models or devising new architectures for better generating and translation performance, are potentially fruitful. For theorists, our framework for analysis can be considered as an attempt to understand the far-more complicated mechanism behind adversarial learning models in a specific context. More exciting theoretical results based on our theoretical framework awaits further dedications.

\newpage

\newpage

\section*{Acknowledgements}
  The authors would like to thank the anonymous referees for
  their valuable comments and helpful suggestions. This work is funded by the National Program on Key Basic Research (NO. 2015CB358800).

\bibliography{main}

\begin{thebibliography}{36}
\providecommand{\natexlab}[1]{#1}
\providecommand{\url}[1]{\texttt{#1}}
\expandafter\ifx\csname urlstyle\endcsname\relax
  \providecommand{\doi}[1]{doi: #1}\else
  \providecommand{\doi}{doi: \begingroup \urlstyle{rm}\Url}\fi

\bibitem[Amari \& Nagaoka(2007)Amari and Nagaoka]{amari2007methods}
Amari, S.-i. and Nagaoka, H.
\newblock \emph{Methods of information geometry}, volume 191.
\newblock American Mathematical Soc., 2007.

\bibitem[Arjovsky \& Bottou(2017)Arjovsky and Bottou]{Arjovsky2017TowardsPM}
Arjovsky, M. and Bottou, L.
\newblock Towards principled methods for training generative adversarial
  networks.
\newblock \emph{CoRR}, abs/1701.04862, 2017.

\bibitem[Arjovsky et~al.(2017)Arjovsky, Chintala, and
  Bottou]{Arjovsky2017WassersteinGA}
Arjovsky, M., Chintala, S., and Bottou, L.
\newblock Wasserstein generative adversarial networks.
\newblock In \emph{ICML}, 2017.

\bibitem[Arora et~al.(2017)Arora, Ge, Liang, Ma, and
  Zhang]{Arora2017GeneralizationAE}
Arora, S., Ge, R., Liang, Y., Ma, T., and Zhang, Y.
\newblock Generalization and equilibrium in generative adversarial nets (gans).
\newblock In \emph{ICML}, 2017.

\bibitem[Aumann(1989)]{aumann1989game}
Aumann, R.~J.
\newblock Game theory.
\newblock In \emph{Game Theory}, pp.\  1--53. Springer, 1989.

\bibitem[Belkin et~al.(2006)Belkin, Niyogi, and
  Sindhwani]{Belkin2006ManifoldRA}
Belkin, M., Niyogi, P., and Sindhwani, V.
\newblock Manifold regularization: A geometric framework for learning from
  labeled and unlabeled examples.
\newblock \emph{Journal of Machine Learning Research}, 7:\penalty0 2399--2434,
  2006.

\bibitem[Boyd \& Vandenberghe(2004)Boyd and Vandenberghe]{boyd2004convex}
Boyd, S. and Vandenberghe, L.
\newblock \emph{Convex optimization}.
\newblock Cambridge university press, 2004.

\bibitem[Cameron(1999)]{cameron1999permutation}
Cameron, P.~J.
\newblock \emph{Permutation groups}, volume~45.
\newblock Cambridge University Press, 1999.

\bibitem[Choi et~al.(2017)Choi, Choi, Kim, Ha, Kim, and Choo]{choi2017stargan}
Choi, Y., Choi, M., Kim, M., Ha, J.-W., Kim, S., and Choo, J.
\newblock Stargan: Unified generative adversarial networks for multi-domain
  image-to-image translation.
\newblock \emph{arXiv preprint arXiv:1711.09020}, 2017.

\bibitem[Chung(2001)]{chung2001course}
Chung, K.~L.
\newblock \emph{A course in probability theory}.
\newblock Academic press, 2001.

\bibitem[Dai et~al.(2008)Dai, Chen, Xue, Yang, and Yu]{Dai2008TranslatedLT}
Dai, W., Chen, Y., Xue, G.-R., Yang, Q., and Yu, Y.
\newblock Translated learning: Transfer learning across different feature
  spaces.
\newblock In \emph{NIPS}, 2008.

\bibitem[Denton et~al.(2015)Denton, Chintala, Szlam, and
  Fergus]{Denton2015DeepGI}
Denton, E.~L., Chintala, S., Szlam, A., and Fergus, R.
\newblock Deep generative image models using a laplacian pyramid of adversarial
  networks.
\newblock In \emph{NIPS}, 2015.

\bibitem[Gatys et~al.(2016)Gatys, Ecker, and Bethge]{Gatys2016ImageST}
Gatys, L.~A., Ecker, A.~S., and Bethge, M.
\newblock Image style transfer using convolutional neural networks.
\newblock \emph{2016 IEEE Conference on Computer Vision and Pattern Recognition
  (CVPR)}, pp.\  2414--2423, 2016.

\bibitem[Goodfellow et~al.(2014)Goodfellow, Pouget-Abadie, Mirza, Xu,
  Warde-Farley, Ozair, Courville, and Bengio]{Goodfellow2014GenerativeAN}
Goodfellow, I.~J., Pouget-Abadie, J., Mirza, M., Xu, B., Warde-Farley, D.,
  Ozair, S., Courville, A.~C., and Bengio, Y.
\newblock Generative adversarial nets.
\newblock In \emph{NIPS}, 2014.

\bibitem[Isola et~al.(2017)Isola, Zhu, Zhou, and
  Efros]{Isola2017ImagetoImageTW}
Isola, P., Zhu, J.-Y., Zhou, T., and Efros, A.~A.
\newblock Image-to-image translation with conditional adversarial networks.
\newblock \emph{2017 IEEE Conference on Computer Vision and Pattern Recognition
  (CVPR)}, pp.\  5967--5976, 2017.

\bibitem[Jost(2008)]{jost2008riemannian}
Jost, J.
\newblock \emph{Riemannian geometry and geometric analysis}.
\newblock Springer Science \& Business Media, 2008.

\bibitem[Kataoka et~al.(2017)Kataoka, Matsubara, and
  Uehara]{Kataoka2017AutomaticMC}
Kataoka, Y., Matsubara, T., and Uehara, K.
\newblock Automatic manga colorization with color style by generative
  adversarial nets.
\newblock \emph{2017 18th IEEE/ACIS International Conference on Software
  Engineering, Artificial Intelligence, Networking and Parallel/Distributed
  Computing (SNPD)}, pp.\  495--499, 2017.

\bibitem[Lee(2010)]{lee2010introduction}
Lee, J.
\newblock \emph{Introduction to topological manifolds}, volume 940.
\newblock Springer Science \& Business Media, 2010.

\bibitem[Lei et~al.(2017)Lei, Su, Cui, Yau, and Gu]{Lei2017AGV}
Lei, N., Su, K., Cui, L., Yau, S.-T., and Gu, X.
\newblock A geometric view of optimal transportation and generative model.
\newblock \emph{CoRR}, abs/1710.05488, 2017.

\bibitem[Lu et~al.(1998)Lu, Fainman, and Hecht-Nielsen]{Lu1998ImageM}
Lu, H., Fainman, Y., and Hecht-Nielsen, R.
\newblock Image manifolds.
\newblock 1998.

\bibitem[Luenberger(1979)]{luenberger1979introduction}
Luenberger, D.~G.
\newblock \emph{Introduction to dynamic systems: theory, models, and
  applications}, volume~1.
\newblock Wiley New York, 1979.

\bibitem[{\L}ukaszyk(2004)]{lukaszyk2004new}
{\L}ukaszyk, S.
\newblock A new concept of probability metric and its applications in
  approximation of scattered data sets.
\newblock \emph{Computational Mechanics}, 33\penalty0 (4):\penalty0 299--304,
  2004.

\bibitem[Mirza \& Osindero(2014)Mirza and Osindero]{Mirza2014ConditionalGA}
Mirza, M. and Osindero, S.
\newblock Conditional generative adversarial nets.
\newblock \emph{CoRR}, abs/1411.1784, 2014.

\bibitem[Qi et~al.(2017)Qi, Zhang, and Hu]{Qi2017GlobalVL}
Qi, G.-J., Zhang, L., and Hu, H.
\newblock Global versus localized generative adversarial nets.
\newblock \emph{CoRR}, abs/1711.06020, 2017.

\bibitem[Reinhard et~al.(2001)Reinhard, Ashikhmin, Gooch, and
  Shirley]{Reinhard2001ColorTB}
Reinhard, E., Ashikhmin, M., Gooch, B., and Shirley, P.
\newblock Color transfer between images.
\newblock \emph{IEEE Computer Graphics and Applications}, 21:\penalty0 34--41,
  2001.

\bibitem[Rudin(2010)]{Rudin2010REALAC}
Rudin, W.
\newblock Real and complex analysis real and complex analysis third edition.
\newblock 2010.

\bibitem[Saito et~al.(2018)Saito, Takamichi, and
  Saruwatari]{Saito2018StatisticalPS}
Saito, Y., Takamichi, S., and Saruwatari, H.
\newblock Statistical parametric speech synthesis incorporating generative
  adversarial networks.
\newblock \emph{IEEE/ACM Transactions on Audio, Speech, and Language
  Processing}, 26:\penalty0 84--96, 2018.

\bibitem[Sontag(1998)]{Sontag1998VCDO}
Sontag, E.~D.
\newblock Vc dimension of neural networks.
\newblock \emph{NATO ASI Series F Computer and Systems Sciences}, 168:\penalty0
  69--96, 1998.

\bibitem[Vapnik \& Vapnik(1998)Vapnik and Vapnik]{vapnik1998statistical}
Vapnik, V.~N. and Vapnik, V.
\newblock \emph{Statistical learning theory}, volume~1.
\newblock Wiley New York, 1998.

\bibitem[Vayatis \& Azencott(1999)Vayatis and
  Azencott]{Vayatis1999DistributionDependentVB}
Vayatis, N. and Azencott, R.
\newblock Distribution-dependent vapnik-chervonenkis bounds.
\newblock In \emph{EuroCOLT}, 1999.

\bibitem[Villani(2008)]{Villani2010OptimalT}
Villani, C.
\newblock \emph{Optimal transport: old and new}, volume 338.
\newblock Springer Science \& Business Media, 2008.

\bibitem[Wu et~al.(2016)Wu, Zhang, Xue, Freeman, and
  Tenenbaum]{Wu2016LearningAP}
Wu, J., Zhang, C., Xue, T., Freeman, B., and Tenenbaum, J.~B.
\newblock Learning a probabilistic latent space of object shapes via 3d
  generative-adversarial modeling.
\newblock In \emph{NIPS}, 2016.

\bibitem[Zeiler et~al.(2011)Zeiler, Taylor, Sigal, Matthews, and
  Fergus]{Zeiler2011FacialET}
Zeiler, M.~D., Taylor, G.~W., Sigal, L., Matthews, I.~A., and Fergus, R.
\newblock Facial expression transfer with input-output temporal restricted
  boltzmann machines.
\newblock In \emph{NIPS}, 2011.

\bibitem[Zhang et~al.(2017)Zhang, Gan, Fan, Chen, Henao, Shen, and
  Carin]{Zhang2017AdversarialFM}
Zhang, Y., Gan, Z., Fan, K., Chen, Z., Henao, R., Shen, D., and Carin, L.
\newblock Adversarial feature matching for text generation.
\newblock In \emph{ICML}, 2017.

\bibitem[Zheng et~al.(2017)Zheng, Zheng, Yu, Gu, and
  Zheng]{Zheng2017PhototoCaricatureTO}
Zheng, Z., Zheng, H., Yu, Z., Gu, Z., and Zheng, B.
\newblock Photo-to-caricature translation on faces in the wild.
\newblock \emph{CoRR}, abs/1711.10735, 2017.

\bibitem[Zhu et~al.(2016)Zhu, Kr{\"a}henb{\"{u}}hl, Shechtman, and
  Efros]{Zhu2016GenerativeVM}
Zhu, J.-Y., Kr{\"a}henb{\"{u}}hl, P., Shechtman, E., and Efros, A.~A.
\newblock Generative visual manipulation on the natural image manifold.
\newblock In \emph{ECCV}, 2016.

\end{thebibliography}
\bibliographystyle{icml2018}

\end{document}


\maketitle
Detailed proofs for all the theorems, lemmas and propositions omitted from our paper will be given here in a rigorous form. We provide them as a supplementary because we would like the audience of our paper to focus  more on the development of our theory, and limit of space. Our proofs are mainly based on the texts of Chung \cite{chung}, Rudin \cite{rudin} and Nakahara \cite{nakahara}.
\section{Introduction}
[No Theorems or Lemmas]

\section{Preliminaries}

\begin{lma} \label{lma:pm_over_mfd}
Given a smooth manifold $\mathcal{M} = \{(U_i,  \varphi_i)\}_{i=1}^{K}$ with pairwise disjointness and $\{\mu_k\}_{k=1}^{K}$ as the probability measures supported on $\{\varphi_i(U_i)\}_{i=1}^{K}$ correspondingly, a function $\mu_{\mathcal{M}}:\mathcal{B}(\mathcal{M})\to[0,1]$ is defined by
\begin{equation} \label{eq:pm_over_mfd}
	 d\mu_{\mathcal{M}}(s) = \frac{1}{K}\sum_{i=1}^{K}\mathbf{1}_{s\in{U_i}}d\mu_{i}\circ\varphi_i(s)
\end{equation}
Then $\mu_{\mathcal{M}}$ is a probability measure defined on $\mathcal{M}$. 
\end{lma}
\begin{proof}
First claim $\mu_{i}\circ\varphi_i$ is a measure, which comes from the easy observation that for each $i \in [K]$, and countably many disjoint sets $\{A_{n}\}_{n=1}^{\infty} \subset \mathcal{B}(\mathcal{M})$, the Borel sets constructed over $\mathcal{M}$ as a topological space.   
$$
	\varphi_{i}(\cup_{n=1}^{\infty}A_{n}) = \cup_{n=1}^{\infty}\varphi_{i}(A_n)
$$

and since $\varphi_i$ itself is a homeomorphism, which indicates the one-to-one property, we have the disjointness of sets $\{\varphi_i(A_{n})\}_{n=1}^{\infty}$.

Thus from the assumption that $\mu_i$ is a probability measure, the countable additivity of $\mu_{i}\circ\varphi_i$ on $\mathcal{B}(\mathcal{M})$ is thus proved as

$$
	\mu_{i}\circ\varphi_i(\cup_{n=1}^{\infty}A_{n}) = \sum_{i=1}^{\infty} \mu_i(\varphi_{i}(A_n))
$$
, which directly leads to the assertion that $\mu_{i}\circ\varphi_i$ is a measure.

Next, we would like to prove it is indeed a probability measure, which needs to prove the normalization condition.

We directly take integral over the manifold $\mathcal{M}$ with the derivative form of measure $\mu_{\mathcal{M}}$, as is defined.  

\begin{gather}
\int_{\mathcal{M}}d\mu_{\mathcal{M}} \\
\overset{\text{substitute}}{=} \frac{1}{K}\int_{\mathcal{M}}\sum_{i=1}^{K}\mathbf{1}_{s\in{U_i}}d\mu_{i}\circ\varphi_i(s) \\
\overset{\text{exchange}}{=}\frac{1}{K}\sum_{i=1}^{K}\int_{U_i}d\mu_{i}\circ\varphi_i(s) \\
\overset{\text{change of variable}} = \frac{1}{K}\sum_{i=1}^{K}\int_{\varphi_i(U_i)}d\mu_{i}(s) \\
= 1 
\end{gather}

Thus we have checked the normalization condition, which in turn proves the lemma.
\end{proof}

\section{Natural Localization of cWGAN-Loss}

\begin{lma} \label{lma:absorption}
Consider Riemmanian manifold $(\mathcal{N}, \tau)$ with curvature locally bounded above and below, $\tau \in C^{\infty}$ and its induced distance function denoted as $d_\mathcal{N}$, then for any path-independent function $f:\mathcal{N}\times{\mathcal{N}}\to\mathbb{R}^{+}\cup\{0\}$, there exists a Riemmanian metric $\tau^{'}$ on $\mathcal{N}$,  induced by the distance function
$$
d_{\mathcal{N}}^{'}(x,y) = f(x,y)d_\mathcal{N}(x,y)\quad{}\forall{x,y}\in{\mathcal{N}}
$$
\end{lma}
\begin{proof}
The proof is mainly based on a previous result in \cite{nikolaev1983smoothness, nikolaev1999metric}, which asserts certain sufficient conditions for a synthetic distance function $d_{\mathcal{N}}^{'}:\mathcal{N}\times{\mathcal{N}}\to\mathbb{R}^{+}\cup\{0\}$ on Riemmanian manifold $(\mathcal{N}, \tau)$ to be compatible with some Riemmanian metric on $\mathcal{N}$. That is, besides the conditions innate to the manifold
\begin{itemize}
\item curvature locally bounded above and below, i.e. $\forall{s\in\mathcal{N}}, \exists{c_1, c_2}$,  $0< c_1 < c_2 < \infty$ and $c_1 < \Gamma_{ij}^{k}(s) <c_2$.
\item $\tau \in C^{\infty}$, which means it is infinitely differentiable locally. In fact, the assumption can be relaxed to $\tau \in C^{1, \alpha}$, for any $\alpha > 0$.
\end{itemize}, 
the condition imposed on the synthetic distance function is
\begin{itemize}
\item $d^{'}_{\mathcal{N}}$ is a \textit{path-metric}, i.e. $\forall{s_1, s_2}\in\mathcal{N}$, consider the set of paths connecting $s_1, s_2$, that is, the set of curves $\mathcal{P}_{s_1\to{s_2}} = \{p:[0,1]\to\mathcal{M}| p(0)= s_1, p(1) = s_2\}$, there exists an functional $L: \mathcal{P}_{s_1\to{s_2}} \to \mathbb{R}^{+}\cup\{0\}$, s.t. 

$$
	d_{\mathcal{N}}^{'}(s_1, s_2) = \inf_{p\in\mathcal{P}_{s_1\to{s_2}}} L(p)
$$

Thus let us turn back to our case, the synthetic distance function is actually expanded from an existing distance function on $\mathcal{N}$, induced by Riemmanian metric $\tau$. As is well known, the induced distance $d_\mathcal{N}$ itself has the form

$$
	d_\mathcal{N}(s_1, s_2) = \inf_{p\in\mathcal{P}_{s_1\to{s_2}}} L(p)
$$

where L is called the length of curve $p$, defined as 

$$	
	L(p) \doteq \int_{0}^{1}\sqrt{\sum_{i,j}g_{ij}(p(t))\frac{\partial{x^{i}}}{\partial{t}}\frac{\partial{x^{j}}}{\partial{t}}}dt
$$

Since from the assumption that $f(\bullet, \bullet)$ is path-independent, we are able to define the following functional $L_f$ (easy to check its well-definedness),

$$
	L_f(p) \doteq f(p(0), p(1))\text{   }\forall{p\in\mathcal{P}_{s_1\to{s_2}}}
$$

Thus it is obvious that, by constructing $L^{'}$ as 

$$
	L^{'}(p) = L_f(p)\bullet{L(p)}\text{   }\forall{p\in\mathcal{P}_{s_1\to{s_2}}}
$$
, our synthetic distance function $d_\mathcal{N}^{'}$ is a path metric thus induced from some Riemmian metric on $\mathcal{N}$, which finishes our proof.
\end{itemize}
\end{proof}

\subsection{Omitted Steps for Renormalization to Obtain Eq. 14}
After we rearrange $d_\mathcal{N}(G(s), t)d\gamma$ as $d_\mathcal{N}^{'}(G(s), t)d\gamma{'}$, the boundary condition $\int_{\mathcal{M}}\int_{\mathcal{N}}d\gamma^{'}=1$ requires renormalization. By introducing an additional matrix $A\in\mathbf{H}(K)$ s.t. $\mathbf{H}{(K)}\doteq\{A\in\mathbb{R}^{K\times{K}} | \forall{j}\in{K},\sum_{i}A^{ij} = K$; $\forall{i,j}\in[K], A^{ij} \ge 0\}$, the cWGAN-loss $\min_G\mathcal{L}^{'}_{adv}(G)$ can be reformulated as
\begin{gather} \label{eq:equiv_form}
	\min_{G} \min_{A\in \mathbf{H}(K)} \sum_{i=1}^{K}\sum_{j=1}^{K}\int_{U_i}\int_{V_j}A^{ij} d^{'}_{\mathcal{N}}(G(s),t)d\tilde{\mu}_{i}(s)d\tilde{\nu}_{j}(t)
\end{gather}
\begin{proof}
We start from the form, 
\begin{align} \label{eq:interm_form}
\min_{G}\underset{\gamma\in{\Pi(\mu_\mathcal{M}, \nu_\mathcal{N})}}{\inf}\sum_{i=1}^{K}\sum_{j=1}^{K}\int_{U_i}\int_{V_j} d_{\mathcal{N}}(G(s),t)d{\gamma(s,t)}
\end{align}

when we rearrange the form with 
$$
d{\gamma(s,t)} = d{\gamma(t|s)}d{\mu_{\mathcal{M}}(s)} = \Delta(f_{\gamma}(s),t)d{\mu_{\mathcal{M}}(s)}d{\nu_{\mathcal{N}}(t)}
$$, we obtain

\begin{align} \label{eq:interm_form}
\min_{G}\min_{f_\gamma}\sum_{i=1}^{K}\sum_{j=1}^{K}\int_{U_i}\int_{V_j} \Delta(f_{\gamma}(s),t)d_{\mathcal{N}}(G(s),t)d{\mu_{\mathcal{M}}(s)}d{\nu_{\mathcal{N}}(t)}
\end{align}

And since we apply the equivalence of $\min_{G}\min_{f_\gamma}$ and $\min_{G}$, the boundary condition $\int_{\mathcal{M}}\int_{\mathcal{N}}d\gamma^{'}=1$ may be broken. Thus we introduce additional variable $A\in{\mathbf{H}(K)}$ to maintain the normalization condition, as can be checked by

\begin{gather}
\int_{M}  d\gamma^{'} \\
= \frac{1}{K^{2}}\sum_{j=1}^{K}\int_{U_i}(\int_{V_j}\sum_{i=1}^{K}A^{ij}d\nu_{j})d\mu_{i} \\
= 1
\end{gather}

Note here we have applied the formulae for constructed probability measures on manifolds as
\begin{equation} \label{eq:mu_form}
	 d\mu_{\mathcal{M}}(s) \doteq \frac{1}{K}\sum_{i=1}^{K}\mathbf{1}_{s\in{U_i}}d\mu_{i}\circ\varphi_i(s)
\end{equation}
\begin{equation}\label{eq:nu_form}
	 d\nu_{\mathcal{N}}(t) \doteq \frac{1}{K}\sum_{i=1}^{K}\mathbf{1}_{t\in{U_i}}d\nu_{i}\circ\psi_i(t)
\end{equation}

Finally, by inserting the $A^{ij}$ term into the original optimization problem above, we will obtain the final form as follows,

\begin{gather} \label{eq:equiv_form}
	\min_{G} \min_{A\in \mathbf{H}(K)} \sum_{i=1}^{K}\sum_{j=1}^{K}\int_{U_i}\int_{V_j}A^{ij} d^{'}_{\mathcal{N}}(G(s),t)d\tilde{\mu}_{i}(s)d\tilde{\nu}_{j}(t)
\end{gather}
\end{proof}

\begin{thm} \label{thm:natural_localization} [Natural Localization of Adversarial Loss]
For any $p\in\text{Sym(K)}$, the optimization problem below
\begin{gather} \label{eq:constrained_equiv_form}
	\min_{G\in F_p} \min_{A\in\mathbf{H}(K) } \sum_{i=1}^{K}\sum_{j=1}^{K}\int_{U_i}\int_{V_j}A^{ij} d^{'}_{\mathcal{N}}(G(s),t)d\tilde{\mu}_{i}(s)d\tilde{\nu}_{j}(t)
\end{gather}
is equivalent to 
\begin{gather} \label{eq:final_equiv_form}  
\min_{G\in F_p} \sum_{i=1}^{K}\int_{U_i}\int_{V_{p(i)}}d^{'}_{\mathcal{N}}(G(s),t)d\tilde{\mu}_{i}(s)d\tilde{\nu}_{p(i)}(t)
\end{gather}
In other words, the optimal $A^{*}\in\mathbf{H}(K) $ has the closed form as
\begin{gather}
(A^{*})^{ij} = K\delta_{j}^{p(i)}
\end{gather}
where $\delta_{j}^{p(i)}$ is the Kronecker delta function.
\end{thm}
\begin{proof}
Fix $i, j\in[K]$, s.t. $j\neq{p(i)}$ and arbitrary $G\in{F_p}$. We first compare the following two terms 
$$
	T_{\text{non-paired}} = \int_{U_i}\int_{V_j}d^{'}_{\mathcal{N}}(G(s),t)d\tilde{\mu}_{i}(s)d\tilde{\nu}_{j}(t)
$$
and 
$$
	T_{\text{paired}} = \int_{U_i}\int_{V_{p(i)}}d^{'}_{\mathcal{N}}(G(s),t)d\tilde{\mu}_{i}(s)d\tilde{\nu}_{p(i)}(t)
$$
Notice, for any $s\in{U_i}$, $G(s)\in{V_{p(i)}}\cap{V_j} = \varnothing$, which comes from the assumption that $G\in{F_p}$ and $j\neq{p(i)}$, which thus leads to $\forall{t\in{V_{p(i)}}, t^{'}\in{V_j}}$, $d_{\mathcal{N}}(G(s), t)\le d_{\mathcal{N}}(G(s),t^{'})$, according to the compatibility of distance function with the assumed inner-relatedness.

And thus $T_{\text{non-paired}} \ge T_{\text{paired}}$. Then we relieve the fixation of $j$. It is easy to see,
$$
\sum_{j=1}^{K}A^{ij} \int_{U_i}\int_{V_{j}}d^{'}_{\mathcal{N}}(G(s),t)d\tilde{\mu}_{i}(s)d\tilde{\nu}_{p(i)}(t) \ge K\int_{U_i}\int_{V_{p(i)}} d^{'}_{\mathcal{N}}(G(s),t)d\tilde{\mu}_{i}(s)d\tilde{\nu}_{p(i)}(t)
$$
, which is equivalent to say the optimal $(A^{ij})^{*} = K\delta_{p(i)}^{j}$ for each $j \in [K]$.

Similarly, we have for each $A\in\mathbf{H}(K)$, 
$$
\sum_{i=1}^{K}\sum_{j=1}^{K}\int_{U_i}\int_{V_j}A^{ij} d^{'}_{\mathcal{N}}(G(s),t)d\tilde{\mu}_{i}(s)d\tilde{\nu}_{j}(t) \ge \sum_{i=1}^{K}\int_{U_i}\int_{V_{p(i)}}d^{'}_{\mathcal{N}}(G(s),t)d\tilde{\mu}_{i}(s)d\tilde{\nu}_{p(i)}(t)
$$
, which brings the equivalence between optimization problems above.
\end{proof}

\section{Generalization for Conditional GAN}

\begin{thm} \label{thm:rel_generalization}
Consider generator $G:\mathbb{R}^{d}\to\mathbb{R}^{d}$ satisfying Lipschitz condition with constant $M_{G}$ and $\mu_X,\nu_Y$ are probability measures on $\mathbb{R}^{d}$ respectively with $\{x_{i}\}_{i=1}^{n_X}\overset{\text{i.i.d.}}{\sim}\mu_{X}$ and $\{y_{i}\}_{i=1}^{n_Y}\overset{\text{i.i.d.}}{\sim}\nu_{Y}$. 

Assume the classical generalization bound satisfies the following inequality with probability $1-\delta$
\begin{align} \label{eq:class_gen}
		\mathbb{E}_{x\sim\mu_X, y\sim\nu_Y}{\|G(x) - y\|} - \sum_{i=1}^{n_X}\sum_{j=1}^{n_Y}\frac{\|G(x_i) - y_j\|}{n_Xn_Y} < \epsilon_{\text{classical}} 
\end{align}
where $\epsilon_{\text{classical}} \doteq \epsilon(n_X, n_Y, \mu_X, \nu_Y, \delta)$ the upper bound 
and ERM-principle \cite{vapnik1998statistical} is satisfied with $\eta$ (i.e. $\frac{1}{n_Xn_Y}\sum_{i=1}^{n_X}\sum_{j=1}^{n_Y}\|G(x_i) - y_j\|  < \eta$), then G generalizes with $(n_X, n_Y)$ training samples and error $\epsilon_{\text{adv}}$ with probability $1-\delta$, i.e. 
\begin{equation} 
	D_{LK}(G(\hat{\mu}_{X}^{n_X}), \nu_Y) - D_{LK}(\hat{\nu}_{Y}^{n_Y}, \nu_Y) < \epsilon_{\text{adv}}
\end{equation}
if the following condition is satisfied
\begin{equation} \label{eq:suff_gen_cond}
\epsilon_{\text{classical}}  - \epsilon_{\text{adv}} + \eta < D_{LK}(\nu_Y,\hat{\nu}_{Y}^{n_Y}) - M_{G}D_{LK}(\mu_{X}, \hat{\mu}_{X}^{n_X}) 
\end{equation}
\end{thm}
\begin{proof}
Let us start by bounding the term $D_{LK}(G(\hat{\mu}_{X}^{n_X}), \nu_Y)$,

\begin{gather}
D_{LK}(G(\hat{\mu}_{X}^{n_X}), \nu_Y) \\
\overset{\text{by def.}}{=} \int_{\mathbb{R}^{d}}\int_{\mathbb{R}^{d}} \|G(x) - y\| d\hat{\mu}_{X}^{n_X}(x)d\nu_Y(y) \\
\overset{\text{norm ineq.}}{\le} \int\int\|G(x) - G(x^{'})\| d\hat{\mu}_{X}^{n_X}(x) d\mu_{X}(x^{'}) +  \int\int \|G(x) - y\| d\mu_{X}(x)d\nu_Y(y) \\
\overset{\text{Lip.}}{\le} M_G\int\int\|(x) - x^{'}\| d\hat{\mu}_{X}^{n_X}(x) d\mu_{X}(x^{'}) + \int\int \|G(x) - y\| d\mu_{X}(x)d\nu_Y(y) \\
\overset{\text{gen. bound, with probability 1-$\delta$}}{\le} M_{G}D_{LK}(\mu_{X}, \hat{\mu}_{X}^{n_X}) + \sum_{i=1}^{n_X}\sum_{j=1}^{n_Y}\frac{\|G(x_i) - y_j\|}{n_Xn_Y} + \epsilon_{\text{classical}} \\
\overset{\text{ERM}}{\le}  M_{G}D_{LK}(\mu_{X}, \hat{\mu}_{X}^{n_X}) + \eta + \epsilon_{\text{classical}} 
\end{gather}

And the definition of generation in adversarial learning sense requires  
$$
	D_{LK}(G(\hat{\mu}_{X}^{n_X}), \nu_Y) - D_{LK}(\hat{\nu}_{Y}^{n_Y}, \nu_Y) < \epsilon_{\text{adv}}
$$

By direct inserting the last expressions during the estimation above, we have obtained the generic inequality to guarantee generalization sufficiently,
$$
	\epsilon_{\text{classical}}  - \epsilon_{\text{adv}} + \eta < D_{LK}(\nu_Y,\hat{\nu}_{Y}^{n_Y}) - M_{G}D_{LK}(\mu_{X}, \hat{\mu}_{X}^{n_X}) 
$$
\end{proof}

\section{Benefits of Localization and Conditions of Generalization}

\begin{prop}
Consider the probability measure underlying the global task as $\mu_X = \frac{1}{K}\sum_{i=1}^{K}\mu_i$ and $\nu_Y = \frac{1}{K}\sum_{i=1}^{K}\nu_i$ in Euclidean sense and
\begin{equation}
		\epsilon^{local}_{adv} = \frac{1}{K}\sum_{i=1}^{K}{D}_{LK}(\mu_{i},\hat{\mu}_{i}^{m})
\end{equation}
\begin{align}
	\epsilon^{global}_{adv} ={D}_{LK}(\frac{1}{K}\sum_{i=1}^{K}\mu_i, \hat{\mu}_{X}^{Km})
\end{align},  
if the compatibility with inner-relatedness ($\forall{i,j\in[K]}, D_{LK}(\mu_i,\mu_j) \ge D_{LK}(\mu_i, \mu_i)$) is satisfied, then
$$
	\epsilon^{local}_{adv} < \epsilon^{global}_{adv}
$$
\end{prop}
\begin{proof} 
First let us consider the situation when $m\to\infty$, which correspondingly leads to $\hat{\mu}_{X}^{Km} \to \frac{1}{K}\sum_{i=1}^{K}\mu_i$ and $\forall{i}\in[K]$, $\hat{\mu}_{i}^{m} \to \mu_i$.

Thus by honestly inserting the term into the definition of $D_{LK}$, we have
\begin{gather}
 \epsilon^{global}_{\text{adv, $n_X\to\infty$}} = D_{LK}(\frac{1}{K}\sum_{i=1}^{K}\mu_i, \frac{1}{K}\sum_{i=1}^{K}\mu_i) \\
= \frac{1}{K^2}\sum_{i=1}^{K}D_{LK}(\mu_i, \mu_i) + \frac{1}{K(K-1)}\sum_{i<j\in[K]}D_{LK}(\mu_i, \mu_j) \\
> \frac{1}{K}\sum_{i=1}^{K}D_{LK}(\mu_i, \mu_i) = \epsilon^{local}_{\text{adv, $n_X\to\infty$}}
\end{gather}
The last inequality comes from the observation that, $\forall{i,j\in[K]}$
$$
D_{LK}(\mu_i, \mu_i) = \|\mu_i - \mu_i\| + 2tr(\Sigma_{\mathcal{M}}) \le \|\mu_i - \mu_j\| + 2tr(\Sigma_{\mathcal{M}}) = D_{LK}(\mu_i, \mu_j)
$$

Next, we would like to consider the case for arbitrary $m$ and the inequality with corresponding optimal empirical estimators $\{\hat{\mu}_i^{m}\}_{i=1}^{K}$. 
Thus with $Kn$ samples, the optimal estimator for the global distribution as a mixture of gaussians with the mixture coefficients priorly known is $\frac{1}{K}\sum_{i=1}^{K}\hat{\mu}_i^{m}$. With 
a similar procedure as above, 

\begin{gather}
 \epsilon^{global}_{\text{adv}} = D_{LK}(\frac{1}{K}\sum_{i=1}^{K}\mu_i, \frac{1}{K}\sum_{i=1}^{K}\hat{\mu}_i^{m}) \\
= \frac{1}{K^2}\sum_{i=1}^{K}D_{LK}(\mu_i, \hat{\mu}_i^{m}) + \frac{1}{K(K-1)}\sum_{1\le{i}<j\le{K}}D_{LK}(\mu_i, \hat{\mu}_j^{m}) \\
\ge  \epsilon^{local}_{\text{adv}}
\end{gather}

with the following observation

\begin{gather}
    D_{LK}(\mu_i, \hat{\mu}_j^{m}) = \|\mu_i - \hat{\mu}_{j} + \hat{\mu}_{j} -  \hat{\mu}_j^{m}\| \\
    \overset{\text{consider $\mu_i - \hat{\mu}_{j}$ white noise}}{=}  \|\mu_i - \hat{\mu}_{j}\| + \|\hat{\mu}_{j} -  \hat{\mu}_j^{m}\|  \\
    > \|\hat{\mu}_{j} -  \hat{\mu}_j^{m}\| = D_{LK}(\mu_j, \hat{\mu}_j^{m})
\end{gather}
\end{proof}

\begin{lma} \label{lma:norm_equiv}
$\forall{i}\in[K]$, consider a measureable mapping $\tilde{f}:U_i\to{V_i}$ with $f \doteq \psi_{i}\circ{\tilde{f}}\circ\varphi_{i}^{-1}$ satisfies Lipschitz condition, then $\int_{U_i}\int_{V_i}d^{'}_{\mathcal{N}}(\tilde{f}(s),t)d\tilde{\mu}_{i}(s)d\tilde{\nu}_{i}(t) \simeq D_{LK}(f(\mu_i), \nu_i)$, i.e. there exists constants $0< C_l < C_u < \infty$ such that 
\begin{equation}
	C_l < \frac{\int_{U_i}\int_{V_i}d^{'}_{\mathcal{N}}(\tilde{f}(s),t)d\tilde{\mu}_{i}(s)d\tilde{\nu}_{i}(t)}{D_{LK}(f(\mu_i), \nu_i)} < C_u
 \end{equation}
\end{lma}
\begin{proof} 
With the measureability of $\tilde{f}$ and smoothness of $\varphi_i, \psi_i$, the induced mapping $\tilde{\nu}^{'} = \frac{1}{\tilde{E}}\tilde{f}(\mu_i)$ and $\nu^{'} = \frac{1}{E}(\psi_i\circ\tilde{f})(\mu_i)$ are also probability measures respectively on $\tilde{f}(U_i)\subset{V_i}$ and $(\psi_i\circ\tilde{f})(U_i)\subset{\psi_i(V_i)}$ (with $\tilde{E},E$ some normalizing factor).

Observe the following bounds, which comes from the inclusion relations above,
$$
	\int_{\tilde{f}(U_i)}\int_{V_i}d^{'}_{\mathcal{N}}(\tilde{f}(s),t)d\tilde{\mu}_{i}(s)d\tilde{\nu}_{i}(t) \le \int_{V_i}\int_{V_i}d^{'}_{\mathcal{N}}(t^{'},t)d\tilde{\nu}^{'}(t^{'})d\tilde{\nu}_{i}(t) 
$$

$$
	\int_{(\psi_i\circ\tilde{f})(U_i)}\int_{\psi_i(V_i)}\|f(s) - t\|d\mu_i(s)d\nu_{i}(t) \le \int_{\psi_i(V_i)}\int_{\psi_i(V_i)}\|t^{'} - t\|d\nu^{'}(t^{'})d\nu_{i}(t) 
$$

With the Lipschitz condition of $f$, it can be asserted that $supp(\tilde{\nu}^{'})$ and $supp(\nu^{'}(t^{'}))$ is bounded by a finite disk respectively on $V_i, \psi_i(V_i)$. Together with the gaussian assumption, 
we have $tr(\Sigma_\mathcal{M})$, $tr(\Sigma_\mathcal{N}) < \infty$, which leads to the boundedness of $supp(\nu_{i})$, $supp(\tilde{\nu}_{i})$ as well.

Thus we have
$$
\int_{V_i}\int_{V_i}d^{'}_{\mathcal{N}}(t^{'},t)d\tilde{\nu}^{'}(t^{'})d\tilde{\nu}_{i}(t) < \infty
$$
$$
\int_{\psi_i(V_i)}\int_{\psi_i(V_i)}\|t^{'} - t\|d\nu^{'}(t^{'})d\nu_{i}(t) < \infty
$$

With the finiteness of right side, we are able to claim
$$
	\int_{\tilde{f}(U_i)}\int_{V_i}d^{'}_{\mathcal{N}}(\tilde{f}(s),t)d\tilde{\mu}_{i}(s)d\tilde{\nu}_{i}(t) \simeq \int_{V_i}\int_{V_i}d^{'}_{\mathcal{N}}(t^{'},t)d\tilde{\nu}^{'}(t^{'})d\tilde{\nu}_{i}(t) 
$$
$$
	\int_{(\psi_i\circ\tilde{f})(U_i)}\int_{\psi_i(V_i)}\|f(s) - t\|d\mu_i(s)d\nu_{i}(t) \simeq \int_{\psi_i(V_i)}\int_{\psi_i(V_i)}\|t^{'} - t\|d\nu^{'}(t^{'})d\nu_{i}(t) 
$$
, which directly leads to the lemma since $\frac{\int_{V_i}\int_{V_i}d^{'}_{\mathcal{N}}(t^{'},t)d\tilde{\nu}^{'}(t^{'})d\tilde{\nu}_{i}(t) }{\int_{\psi_i(V_i)}\int_{\psi_i(V_i)}\|t^{'} - t\|d\nu^{'}(t^{'})d\nu_{i}(t) } < \infty$
\end{proof}

\begin{thm} \label{thm:generalization_cond}
Under the assumptions above, consider a generator $G \in F_e$ and a hypothesis space $\mathcal{H}$ with VC-dimension bound by constant $\Lambda$. Assume for each $i\in[K]$, the restriction of G to a pair of charts $f_i \doteq G_{\downarrow{(U_i, V_i)}}\in\mathcal{H}$ with $(\psi_{i}\circ{G}\circ\varphi_{i}^{-1})$ satisfies Lipschitz condition with constant $M_G$, then G generalizes globally with $(Kn, Km)$ samples only if the following inequality is satisfied with probability $1- C(\epsilon, \Lambda)(nm\epsilon^2)^{\tau(\Lambda)}e^{-nm\alpha\epsilon^{2}}$,  
\begin{gather} \label{eq:generalization_cond}
	\epsilon + \frac{1}{nm}\max\{\sum_{i=1}^{n}\sum_{j=1}^{m}d_\mathcal{N}(G(s_{k}^{i}), t_{k}^{j})\}_{k=1}^{K} < \nonumber \\
   	\frac{1}{\sqrt{m}}\sqrt{\text{tr}(\Sigma_N)} + 2\text{tr}(\Sigma_N) - M_G(\frac{1}{\sqrt{n}}\sqrt{\text{tr}(\Sigma_M)} + 2\text{tr}(\Sigma_M))
\end{gather}
where $C(\epsilon, \Lambda)$ and $\tau(\Lambda)$ are positive functions independent from $n,m$ and $\alpha\in[1,2]$ a constant.   
\end{thm}
\begin{proof} 
For $K$ independent local tasks, with Lma. \ref{lma:norm_equiv}, the global generalization condition in Thm. \ref{thm:rel_generalization} will thus be written as 
$$
	\max\{\epsilon_{\text{classical}}^{i}  - \epsilon_{\text{adv}}^{i} + \eta^{i}\}_{i=1}^{K} < \min\{D_{LK}(\nu_i,\hat{\nu}_{i}^{n}) - M_{G}D_{LK}(\mu_{i}, \hat{\mu}_{i}^{m}) \}_{i=1}^{K}
$$
, which serves as a sufficient condition (note it is not a necessary condition) in the worst case.

We would like to consider the situation when $\epsilon_{\text{adv}}^{i} = 0$ and since the classical generalization error is equivalent with the assumption that the observed samples on each pair of charts are identical, we can reformulate the inequality as

$$
\epsilon_{\text{classical}} + \max\{\eta^{i}\}_{i=1}^{K} < \min\{D_{LK}(\nu_i,\hat{\nu}_{i}^{n}) - M_{G}D_{LK}(\mu_{i}, \hat{\mu}_{i}^{m}) \}_{i=1}^{K}
$$

Apply the result from \cite{vayatis1999DistributionDependentVB}, we could bound the left side by $\epsilon$ with probability $1- C(\epsilon, \Lambda)(nm\epsilon^2)^{\tau(\Lambda)}e^{-nm\alpha\epsilon^{2}}$, that is
$$
\epsilon_{\text{classical}} + \max\{\eta^{i}\}_{i=1}^{K} < \epsilon + \frac{1}{nm}\max\{{\sum_{i=1}^{n}\sum_{j=1}^{m}d_\mathcal{N}(G(s_{k}^{i}), t_{k}^{j})}\}_{k=1}^{K}
$$

The next step is to deal with the right side, with a honest calculation, we could deduce

\begin{gather}
\min\{D_{LK}(\nu_i,\hat{\nu}_{i}^{n}) - M_{G}D_{LK}(\mu_{i}, \hat{\mu}_{i}^{m}) \}_{i=1}^{K} \\
= \min\{\mathbb{E}\|\nu_i - \nu_i^{n}\| - M_{G}\mathbb{E}\|\mu_i - \mu_i^{n}\|\}_{i=1}^{K}  +  2\text{tr}(\Sigma_N)  - 2\text{tr}(\Sigma_M)
\end{gather}

In order to write the first minimization term in a closed form, we use the following theorem from the theory of information geometry of Amari \cite{amari2007methods}

\begin{thmm} \cite[Theorem 4.4]{amari2007methods}
The mean square error of a biased-corrected first-order efficient estimator is given asymptotically by the expansion (with N observed samples):
$$
	\mathbb{E}[(\hat{u}^{a} - u^{a})(\hat{u}^{b} - u^{b})] = \frac{1}{N}g^{ab} + O(\frac{1}{N^2})
$$
\end{thmm}
where $g^{ab}$ denotes the Fisher metric on the manifold constructed from a parametrized family of probability.  

We thus apply such an estimation to figure out $\mathbb{E}\|\nu_i - \nu_i^{n}\|$ and $\mathbb{E}\|\mu_i - \mu_i^{n}\|$. As is well known, the matrix of fisher metric for a gaussian $\mathcal{N}(x, \Sigma)$ is directly $\Sigma$, the covariance matrix itself. 

By observing $\mathbb{E}\|\nu_i - \hat{\nu}_i^{n}\| = \sqrt{tr(\mathbb{E}[(\nu_i - \hat{\nu}_i^{n})(\nu_i - \hat{\nu}_i^{n})^{T}])}$ and $\mathbb{E}\|\mu_i - \hat{\mu}_i^{m}\| = \sqrt{tr(\mathbb{E}[(\mu_i - \hat{\mu}_i^{n})(\mu_i - \hat{\mu}_i^{n})^{T}])}$, we have (with $O(N^{-2})$ term omitted)

$$
\min\{D_{LK}(\nu_i,\hat{\nu}_{i}^{n}) - M_{G}D_{LK}(\mu_{i}, \hat{\mu}_{i}^{m}) \}_{i=1}^{K} \\
= \frac{1}{\sqrt{m}}\sqrt{\text{tr}(\Sigma_N)} + 2\text{tr}(\Sigma_N) - M_G(\frac{1}{\sqrt{n}}\sqrt{\text{tr}(\Sigma_M)} + 2\text{tr}(\Sigma_M))
$$
, which thus gives the condition for generalization above. 
\end{proof}